%% file: jcvx.tex
\documentclass{article}

\usepackage{microtype}
\usepackage{graphicx}
\usepackage{subfigure}
\usepackage{booktabs} 

\usepackage{hyperref}

\usepackage{amssymb, amsmath, amsthm}
\usepackage{shortcuts}
\usepackage{cleveref}
\usepackage[square,sort,comma,numbers]{natbib}
\usepackage{float}

\graphicspath{{images/},{prebuiltimages/}}

\title{Screening Rules and its Complexity for Active Set Identification}
\author{Eugene Ndiaye\footnote{Riken AIP, Japan. Correspondence to: eugene.ndiaye@riken.jp} \qquad Olivier Fercoq\footnote{LTCI, T\'el\'ecom Paris, Institut Polytechnique de Paris, Palaiseau, France.} \qquad Joseph Salmon\footnote{IMAG,Univ Montpellier, CNRS, Montpellier, France}}

\begin{document}
\sloppy
\date{}
\maketitle

\begin{abstract}
Screening rules were recently introduced as a technique for explicitly identifying active structures such as sparsity, in optimization problem arising in machine learning. This has led to new methods of acceleration based on a substantial dimension reduction.
We show that screening rules stem from a combination of natural properties of subdifferential sets and optimality conditions, and can hence be understood in a unified way.
Under mild assumptions, we analyze the number of iterations needed to identify the optimal active set for any converging algorithm. We show that it only depends on its convergence rate.
\end{abstract}

\input{subfiles/introduction}
\input{subfiles/framework}
\input{subfiles/complexity}
\input{subfiles/acceleration_strategies}
\input{subfiles/discussions}

\bibliography{jcvx}
\bibliographystyle{plain}

\end{document}

%% file: subfiles/introduction.tex

\section{Introduction}

In learning problems involving a large number of variables, sparse models such as Lasso and Support Vector Machines (SVM) allow to select the most important variables.
For instance, the Lasso estimator depends only on a subset of features that have a maximal absolute correlation with the residual; whereas the SVM classifier depends only on a subset of sample (the support vectors) that characterize the margin.
The remaining features/variables have no contribution to the optimal solution.
Thus, early detection of those non influential variables may lead to significant simplifications of the problem, memory and computational resources saving. 
Some noticeable examples are the \emph{facial reduction} preprocessing steps used for accelerating the linear programming solvers \citep{Markowitz56, Brearley_Mitra_William75} and conic programming \citep{Borwein_Wolkowicz81}, we refer to \citep{Meszaros_Suhl03, Drusvyatskiy_Wolkowicz17} for recent reviews. Another applications can also be found in \citep{Michelot86} for projecting onto the simplex and $\ell_1$ ball in \citep{Condat16} or data preprocessing before application of statistical methods \citep{Fan_Lv2008} in high dimensional settings.
In an optimization problem, screening rules eliminate the variables that have no influence on the set of optimal coefficients.
Recently, \citep{ElGhaoui_Viallon_Rabbani12} have introduced the \emph{safe screening rules}, which ignore non-active variables in Lasso or non-support vectors for SVM, without false exclusions.
It is a versatile optimization technique useful for many machine learning tasks, see \cite{Xiang_Xu_Ramadge11, Ogawa_Suzuki_Takeuchi13, Wang_Zhou_Liu_Wonka_Ye14, Fercoq_Gramfort_Salmon15, Johnson_Guestrin15, Shibagaki_Karasuyama_Hatano_Takeuchi16, Raj_Olbrich_Gartner_Scholkpof_Jaggi16} to name a few.
However, in the existing safe screening studies, a case by case safe screening rule is developed for each problem, and there is no unified understanding for which class of optimization problems safe screening are possible and how efficient the screening can be.
We summarize our contributions as follow:
\begin{itemize}
\item We provide a simple setting for explicitly identifying optimal active sets in convex composite optimization problems with separable regularization.
Our approach combines a natural property of the subdifferential set with optimality conditions. It allows us to subsume the previously introduced screening rules in a unified framework.

\item Relying on a recent work on the convergence of the duality gap \cite{Dunner_Forte_Takac_Jaggi16}, we can easily provide (for safe screening of features using smooth losses or for safe screening of observations with a strongly-convex penalty) an algorithm-independent complexity analysis of the (finite) active set identification.
Then, the latter holds for any converging optimization scheme as soon as it is endowed with screening rules.

\item We discuss some acceleration strategies based on a combination of safe and relaxed safe rules. Several popular strategies such as strong rules \cite{Tibshirani_Bien_Friedman_Hastie_Simon_Tibshirani12} and some recent working sets methods \cite{Johnson_Guestrin15, Johnson_Guestrin16, Massias_Vaiter_Gramfort_Salmon19} can then be generalized to a larger set of optimization problems.

\end{itemize}


\paragraph{Notation.}
Given a proper, closed and convex function $f: \bbR^n \to \bbR \cup \{+\infty\}$, we denote $\dom f = \{x \in \bbR^n: f(x) < +\infty\}$. The Fenchel-Legendre conjugate of $f$ is the function $f^*:\bbR^n \to \bbR \cup \{+\infty\}$ defined by
$$f^*(x^*) = \sup_{x \in \dom f} \langle x^* , x \rangle - f(x) \enspace.$$ 
The subdifferential of a proper function $f$ at $x$ is the set
\begin{equation}\label{eq:definition_subdifferential}
\partial f(x) = \{v \in \bbR^n: f(z) \geq f(x) + \langle v, z - x \rangle, \forall z \in \bbR^n \} \enspace.
\end{equation}
The support function of a nonempty set $C$ is defined as $\sfunc{C} (x) = \sup_{c \in C} \langle c,x \rangle$. If $C$ is closed, convex and contains $0$, we define its polar function as $\sfunc{C}^{\circ}(x^*) = \sup_{\sfunc{C} (x) \leq 1} \langle x^*, x \rangle$.
The interior (resp. boundary) of a set $C$ is denoted $\interior C$ (resp. $\bd C$). 
We denote by $[T]$ the set $\{1, \ldots, T\}$ for any non zero integer $T$.
The vector of observations is $y \in \bbR^n$ and the design matrix $X= [x_1,\dots,x_n]^\top = [X_1,\dots, X_p] \in \bbR^{n\times p}$ has $n$ observations row-wise, and $p$ features (column-wise). A group of features is a subset $g \subset [p]$ and $|g|$ is its cardinality.  The set of groups is denoted by $\mathcal{G}$. We denote by $\beta_g$ the vector in $\bbR^{|g|}$ which is the restriction of $\beta$ to the indices in $g$. We also use the notation $X_g \in \bbR^{n \times |g|}$ to refer to the sub-matrix of X assembled from the columns with indices $j \in g$.


%% file: subfiles/framework.tex

\section{General framework}
\label{sec:general_framework}

We consider composite optimization problems involving a sum of a data fitting function plus a separable group regularization function $\Omega(\beta) = \sum_{g \in \mathcal{G}} \Omega_g(\beta_g)$ that enforces specific structure such as sparsity in the optimal solutions:
\begin{equation}\label{eq:general_optim}
\hat\beta \in \argmin_{\beta \in \bbR^p} f(X\beta) + \Omega(\beta) = P(\beta) \enspace.
\end{equation}
We will assume that the functions $f$ and $\Omega$ are proper, lower-semicontinuous and convex. We also assume that $\mathrm{Im}(X) \cap \dom(f)$ is non empty, the primal problem admits a solution
 and there exists a vector $\beta$ in $\dom(\Omega)$ such that $f$ is continuous at $X\beta$.
Such a formulation often arises in statistical learning in a context of regularized empirical risk minimization \cite{ShalevShwartz_BenDavid14}.

The main purpose of screening rules is to identify and eliminate the irrelevant features/samples during (or before) an optimization process for solving Problem~\eqref{eq:general_optim}, to reduce the memory and computational footprint.
For instance, many features $X_j$ for some $j$ in $[p]$ are expected to be irrelevant in the Lasso estimator \cite{Tibshirani96}, defined as
$\hat\beta \in \argmin_{\beta\in\bbR^p} \frac{1}{2}\norm{y - X\beta}^2 + \lambda \norm{\beta}_1$ (for some observation vector $y\in \bbR^n$ and regularization parameter $\lambda >0)$.
They correspond to coordinates where $\hat\beta_j = 0$.
Exploiting the known sparsity of the solution, safe screening rules \cite{ElGhaoui_Viallon_Rabbani12} discard features prior to starting a sparse solver.
Computational gains are obtained from the reduction of the dimension.
A similar example is the SVM classifier relying on the hinge loss, defined as
$\hat\beta \in \argmin_{\beta\in\bbR^p} \sum_{i \in [n]}\max(0, 1 - y_i x_{i}^{\top}\beta) +  \frac{\lambda}{2}\norm{\beta}^2$ in which some sample $x_i$ are expected to be irrelevant: $\hat \theta_i = 0$ for some $i$ in $[n]$ where $\hat\theta$ is the associated dual solution.
Identifying irrelevant variables in optimization is not restricted to sparse problems. We show how it is intimately related to the non smoothness of the objective function and present a simple and self contained framework for developing screening rules.

\subsection{Identification rules}

We introduce the main lemma that captures a natural property of the subdifferential~\eqref{eq:definition_subdifferential}, that allows us to obtain a simple and unified presentation of many screening rules recently introduced in the literature.

\begin{lemma}[Separation of subdifferentials] \label{lm:fundamental_lemma}
Let $P$ be a proper function and $z$ such that $\interior \partial P (z) \neq \emptyset$. Then we have $\interior (\partial P (z)) \cap \partial P (z') = \emptyset$ for all $z \neq z'$.
\end{lemma}
\begin{proof}
Let $z'$ such that there exists $g$ in $\interior \partial P (z) \cap \partial P (z')$.
Now $g$ in the open set $\interior \partial P (z)$ implies that there exists $\alpha > 0$ such that $g_{\alpha} := g + \alpha (z' - z) \in \partial P (z)$. Then
\begin{align*}
P(z) &\geq P(z') + \langle g, z - z' \rangle \\
&\geq P(z) + \langle g_{\alpha}, z' - z \rangle + \langle g, z - z' \rangle\\
& = P(z) + \alpha \norm{z' - z}_{2}^{2} \enspace,
\end{align*}
where the first (resp. the second) inequality comes from $g \in \partial P(z')$ (resp. $g_{\alpha} \in \partial P(z)$). Thus $z' = z$.
\end{proof}

\begin{remark}
Without precautions, the interior cannot be replaced by the relative interior in \Cref{lm:fundamental_lemma}. For example, the function $P(z) = |z_1 - z_2|$ defined on $\bbR^2$ does not satisfy the assumption of the Lemma. Indeed, any $z=(z_1,z_2) \in \bbR^{2}$ for $z_1=z_2$, the subdifferential $\partial P(z) = \{(a, -a):\, a \in [-1, 1] \}$ is a one dimensional line segment in $\bbR^2$ with empty interior. Also for non-convex function, note that the Lemma is applicable only when the subdifferential set defined in \Cref{eq:definition_subdifferential} is non-empty.
\end{remark}

The contrapositive of \Cref{lm:fundamental_lemma} leads to a simple identification rule:
\begin{equation}\label{eq:screening_principle}
\interior \partial P (z) \cap \partial P (z') \neq \emptyset \Longrightarrow z = z' \enspace.
\end{equation}
Fermat's rule provides the optimality condition:
\begin{align*}
0 \in \partial P(\hat \beta) = X^\top \partial f(X \hat\beta) + \partial \Omega(\hat\beta) \enspace,
\end{align*}
which is equivalent to
\begin{align}\label{eq:fermat_rule}
X^\top \hat \theta \in \partial \Omega(\hat\beta) \text{ for any } \hat \theta \in -\partial f(X \hat\beta) \enspace.
\end{align}

Suppose now that there exists $\beta_g^\star$ such that $\interior \partial \Omega_g(\beta_{g}^{\star})$ is non empty. In the case $\Omega_g(\beta_g) = \|\beta_g\|$ for instance, we can take $\beta_g^\star = 0$.
For any group of features $g$ in $\mathcal{G}$ such that $X_{g}^{\top} \hat \theta \in \partial \Omega_g(\hat\beta_g)$, and any vector $\beta_{g}^{\star}$  such that $\interior \partial \Omega_g(\beta_{g}^{\star})$ is non empty, we have
\begin{align*}
X_{g}^{\top} \hat \theta \in \interior \partial \Omega_g(\beta_{g}^{\star}) &\Longrightarrow
\interior \partial \Omega_g(\beta_{g}^{\star}) \cap \partial \Omega_g(\hat\beta_g) \neq \emptyset\\
&\overset{\eqref{eq:screening_principle}}{\Longrightarrow} \beta_{g}^{\star} = \hat\beta_g \enspace.
\end{align*}
This relation means that the group of feature $X_g$ is irrelevant and can be discarded in the Problem~\eqref{eq:general_optim} \ie $\hat \beta_g$ is identified to be equal to $\beta_{g}^{\star}$, whenever $X_{g}^{\top} \hat \theta$ belongs to $\interior \partial \Omega_g(\beta_{g}^{\star})$.
However, since $\hat \theta$ depends on the \textbf{unknown} solution $\hat\beta$, this rule is of limited use. Fortunately, it is often possible to construct a set $\mathcal{R} \subset \bbR^{n}$, called a \emph{safe region}, that contains such a $\hat \theta$. This observation leads to the following result.
%
\begin{proposition}[Feature-wise screening rule] \label{prop:feature_screening}
Let $\beta_{g}^{\star}$ such that $\interior \partial \Omega_g(\beta_{g}^{\star}) \neq \emptyset$ and let $\mathcal{R}$ be a set containing $\hat \theta$. Then
\begin{align*}
	X_{g}^{\top} \mathcal{R} \subset \interior \partial \Omega_g(\beta_{g}^{\star}) \Longrightarrow \beta_{g}^{\star} = \hat\beta_g \enspace.
\end{align*}
\end{proposition}

\Cref{prop:feature_screening} provides a general recipe for applying (safe) screening rule in optimization Problem~\eqref{eq:general_optim} without assumption on the algebraic expression of functions $f$ and $\Omega$.

\paragraph{Choice of the safe region $\mathcal{R}$.}
When $f$ is $L$-Lipschitz, for any $v \in -\partial f(X \hat\beta)$ as in \cref{eq:fermat_rule}, we have $\normin{v} \leq L \norm{X}$ (see \citep[Lemma 2.6]{Shalev_12}). The set $\mathcal{R}$ can be taken as a ball of radius $L \norm{X}$. When $f$ is differentiable with Lipschitz gradient, we will see that $\hat \theta$ belongs to a ball centered around any dual estimate $\theta$ and whose radius quantifies the approximation quality.
Thus, the better the estimate $\theta$, the smaller the safe region $\mathcal{R}$, \ie the more efficient the screening rule is.

\paragraph{Dual formulation.}
Similar results naturally hold in the dual space. Under the classical Fenchel-Rockafellar duality \citep[Chapter 31]{Rockafellar97}, the dual formulation of Problem~\eqref{eq:general_optim} reads:
\begin{equation}\label{eq:dual_problem_screening}
 \hat \theta \in \argmax_{\theta \in \bbR^n} -f^{*}(-\theta) - \Omega^{*}(X^{\top}\theta) =  D(\theta) \enspace.
\end{equation}
The solutions $\hat \beta$ in the primal and $\hat \theta$ in the dual are linked by the optimality condition in \Cref{eq:fermat_rule}.
In \Cref{eq:dual_problem_screening}, the role of $f$ and $\Omega$ are flipped and we have the next result.
\begin{proposition}[Sample-wise screening rule] \label{prop:sample_screening}
Suppose that $f(x) = \sum_{i=1}^n f_i(x_i)$.
 Let $\theta^{\star} \in \bbR^n$ be a vector such that for some $i\in[n]$, $\interior \partial f_{i}^{*}(\theta_{i}^{\star}) \neq \emptyset$ and let $\mathcal{R}\subset \bbR^p$ be a set containing a solution $\hat \beta$ of Problem~\eqref{eq:general_optim}.
 Then $$x_{i}^{\top}\mathcal{R} \subset \interior\partial f_{i}^{*}(\theta_{i}^{\star}) \Longrightarrow - \hat \theta_i = \theta_{i}^{\star} \enspace.$$
\end{proposition}
For simplicity, we will mostly restrict the discussion on the primal formulation since the same properties can be recovered by duality.

\subsection{Examples}

To apply the screening rule within the introduced setting, we only have to identify the subdifferential of $\Omega$ and points where it has a non empty interior. We present a few examples popular in machine learning and statistics, as well as other ones commonly met in convex optimization \cite{HiriartUrruty_Lemarechal01}.

\paragraph{Sparsity inducing regularization.}
The Lasso and Elastic net are examples where each group reduces to a single feature $g = \{j\}$ for $j\in [p]$.
In both cases the feature-wise penalties can be written respectively $\Omega_g(\beta_g) = |\beta_j|$, and $\Omega_g(\beta_g) = |\beta_j| + \alpha \beta_{j}^{2}/2$ for $\alpha\geq0$.
For $j \in [p]$ and $\beta_{j}^{\star} = 0$, we have $\partial \Omega_{j}(\beta_{j}^{\star}) = [-1, 1]$ which has a non empty interior equal to $(-1, 1)$.
We also have cases where, $\Omega_g = \norm{\cdot}$ is a norm. It includes examples such as Group Lasso \cite{Yuan_Lin06} and Sparse-Group Lasso \citep{Simon_Friedman_Hastie_Tibshirani13} for instance.
In that case, $\beta_{g}^{\star} = 0 \in \bbR^{|g|}$ and $\partial \Omega_g(\beta_{g}^{\star})$ is equal to the unit ball \wrt to the dual norm of $\Omega_g$, which also have a non empty interior.
These examples also easily extend to gauges and support functions (see \cite{HiriartUrruty_Lemarechal01} for definitions). For instance, Let $Q \in \bbR^{|g| \times |g|}$ be a symmetric positive semi-definite matrix and define $\Omega_g(\beta_g) = \sqrt{\langle Q\beta_g, \beta_g \rangle}$. Then $\beta_{g}^{\star}$ can be any element of $\mathrm{Ker}(Q)$ and $\partial \Omega_g(\beta_{g}^{\star}) = Q^{1/2}\ball(0, 1)$, where $\ball(0, 1)$ is the unit ball for $\norm{\cdot}$ in $\bbR^{|q|}$.

\paragraph{Non negativity constraint:} $\Omega_j(\beta_j) = \iota_{\bbR_+}(\beta_j)$ with $g=\{j\}$. In this case, $\beta_{j}^{\star} = 0$ and $\partial \Omega(\beta_{j}^{\star}) = \bbR_{-}$. This formulation appears in non-negative Least-square, non-negative Lasso and simplex constrained problems.

\paragraph{Box constraint:} $\Omega_j(\beta_j) = \iota_{[a, b]}(\beta_j)$ with $g=\{j\}$. In this case $\beta_{j}^{\star}$ belongs to $\{a, b\}$ and $\partial \Omega_j(\beta_{j}^{\star})=\bbR_{-}$ (resp. $\bbR_{+}$) if $\beta_{j}^{\star}$ is equal to $a$ (resp. $b$).

\paragraph{Hinge:} $\Omega_j(\beta_j) = \max(0, 1 - \beta_j)$ with $g=\{j\}$, we have $\beta_{j}^{\star} = 1$ and $\partial\Omega(\beta_{j}^{\star})= [-1, 0]$ which have a non empty interior $(-1, 0)$. This is used in $\epsilon$-insensitive loss which flattened the $\ell_1$ norm around zero \ie $\Omega_j(\beta_j) = \max(0, |\beta_j| - \epsilon)$ for some $\epsilon \geq 0$.

\paragraph{Distance functions:}
Let $C$ be a closed convex subset of $\bbR^{|g|}$ and $\Omega_g(\beta_g) = \min\{\norm{z - \beta_g}: z \in C\}$. In that case, $\partial\Omega_g(\beta_{g}^{\star}) = N_C(\beta_{g}^{\star}) \cap \ball(0, 1)$ where $N_C(w)$ is the normal cone of $C$ at $w\in \bbR^{|g|}$. Hence, $\beta_{g}^{\star}$ can be any vector in $\bd C$ such that $N_C(\beta_{g}^{\star})$ has a non empty interior.
\paragraph{Piecewise affine functions:}
Let us consider a set of real values $\{r_1, \cdots, r_m\}$ and vectors $\{s_1, \cdots, s_m\}$ in $\bbR^{|g|}$ for an integer $m \geq 2$. One can define $\Omega_g(\beta_g) = \max\{r_j + \langle s_j, \beta_g \rangle : j \in [m]\}$, then $\partial\Omega_g(\beta_{g}^{\star})$ is the convex hull of the set $\{s_j : j \in J(\beta_{g}^{\star})\}$ where $J(w) = \{j \in [m] : \Omega_g(w) = r_j + \langle s_j, w\rangle\}$ and $\beta_{g}^{\star}$ can be chosen as any vector such that the matrix $(s_j)_{j \in J(\beta_{g}^{\star})}$ has full rank. This generalize the previous examples of $\ell_1$ regularization and hinge loss. We refer to \citep[Chapter D]{HiriartUrruty_Lemarechal01} for a further generalization to supremum over a collection of convex function, more sophisticated examples and detailed subdifferential calculus rules.

\section{Explicit active set identification}
\label{sec:Active_Set_Identification}

The main interest of using screening rules in optimization algorithms is to focus the computational efforts on the most important variables (or samples in the dual).
To do so, one needs to explicitly and safely identify parts of the solution vector \ie detect some group $g$ such that we can guarantee that $\hat\beta_g = \beta_{g}^{\star}$ where the latter is such that $\interior \partial \Omega(\beta_{g}^{\star})$ is not empty (which is a necessary condition for identifiability in our framework). This is particularly well suited for proximal (block) coordinate descent method as this type of method can easily ignore non influential (block) coordinates.

Any time a safe region $\mathcal{R}$ is considered for a safe screening test (following \Cref{prop:feature_screening}), one can associate to it a \emph{safe active set} consisting of the features that cannot yet be removed.

\begin{definition}[Feature-wise (safe) active sets]\label{def:active_set}

Let $\beta_{g}^{\star}$ be a vector such that $\interior \partial \Omega_g(\beta_{g}^{\star})$ is non empty. The set of (group) active features at $\beta_{g}^{\star}$ is defined as:
\begin{align*}
\mathcal{A} &:= \mathcal{A}(\hat\theta) = \left\{g \in \mathcal{G}:\; X_{g}^{\top}\hat \theta \notin \interior \partial \Omega_g(\beta_{g}^{\star}) \right\} \enspace,
%
\end{align*}
where $\hat \theta$ is a dual optimal solution in \Cref{eq:dual_problem_screening}.
Moreover, if $\mathcal{R}$ is a safe region, its corresponding set of (group) safe active features at $\beta_{g}^{\star}$ is defined as
\begin{align*}
\mathcal{A}_{\mathcal{R}} &:= \left\{g \in \mathcal{G}:\; X_{g}^{\top}\mathcal{R} \not \subset \interior \partial \Omega_g(\beta_{g}^{\star}) \right\} \enspace.
\end{align*}
The complements (\ie the set of non active groups) of $\mathcal{A}$ and $\mathcal{A}_{\mathcal{R}}$ are denoted by $\mathcal{Z}$ and $\mathcal{Z}_{\mathcal{R}}$.
\end{definition}

\subsection{Computation of screening tests}

The set inclusion tests presented in \Cref{prop:feature_screening,prop:sample_screening} have a simple numerical implementation.
Since the subdifferential $\partial \Omega_g(\beta_{g}^{\star})$ is a closed convex set, for any region $\mathcal{R}$, the screening test
$ X_{g}^{\top}\mathcal{R} \subset \interior \partial \Omega_g(\beta_{g}^{\star}) $
can be evaluated computationally thanks to the following classical lemma that allows to check whether a point belongs to the interior of a closed convex set by means of support function. We recall that for a nonempty set $C$, it is defined as $\sfunc{C} (x) = \sup_{c \in C} \langle c,x \rangle$.

\begin{lemma}[{\cite[Theorem C-2.2.3]{HiriartUrruty_Lemarechal01}}]\label{lm:interior_test}
Let $C_1$ and $C_2$ be nonempty closed convex subset of $\bbR^n$. Then, we have for any direction $d \in \bd\ball(0,1)$,
$$\sfunc{C_1}(d) < \sfunc{C_2}(d) \Longleftrightarrow C_1 \subset \interior C_2 \enspace.$$
\end{lemma}
%

%
By applying \Cref{lm:interior_test} to the closed convex sets $C_1 = X_{g}^{\top} \mathcal{R}$ and $C_2 = \partial \Omega_g(\beta_{g}^{\star})$, we obtain a computational tool for evaluating the screening tests.
\begin{proposition} \label{prop:general_safe_rule}
Let $\mathcal{R}$ be a closed convex set that contains the dual optimal solution $\hat \theta$ in \Cref{eq:dual_problem_screening}. For all group $g$ in $\mathcal{G}$ and for any vector $\beta_{g}^{\star}$ such that $\interior \partial \Omega_g(\beta_{g}^{\star})$ is non empty, $\hat \beta_g = \beta_{g}^{\star}$ whenever for any direction $d\in\bd\ball(0,1)$
\begin{align*}
&\textbf{Screening:}
&\sfunc{\{X_{g}^{\top}\hat \theta\} }(d) < \sfunc{\partial \Omega_g(\beta_{g}^{\star})}(d)\enspace.\\
&\textbf{Safe screening:}
& \sfunc{X_{g}^{\top} \mathcal{R}}(d) < \sfunc{\partial \Omega_g(\beta_{g}^{\star})}(d) \enspace.
\end{align*}
\end{proposition}

The \emph{safe screening} rule consists in removing the $g$-th group from the optimization process whenever the previous test is satisfied, since then $\hat \beta_{g}$ is guaranteed to be equal to $\beta_{g}^{\star}$. Should $\mathcal{R}$ be small enough to screen many groups, one can observe considerable speed-ups in practice as long as the testing can be performed efficiently. Thus a natural goal is to find safe regions as narrow as possible and only cheap computations are needed to check if $X_{g}^{\top}\mathcal{R} \subset \interior \partial \Omega_g(\beta_{g}^{\star})$.
Relying on optimality conditions, one can easily find a set containing the dual optimal solution $\hat \theta$.
For a pair of vector $(\beta, \theta) \in \dom P \times \dom D$, the duality gap is defined as the difference between the primal and dual objectives:
\begin{align*}
\Gap(\beta, \theta) = P(\beta) - D(\theta) \enspace.
\end{align*}
For such $(\beta, \theta)$ pair, weak duality holds: $P(\beta) \geq D(\theta)$, namely
$
P(\beta) - P(\hat\beta) \leq \Gap(\beta, \theta).
$
At optimal values and under mild conditions, we have the strong duality $\Gap(\hat \beta, \hat \theta) = 0$. This allows us to exploit the duality gap as an optimality certificate or as an algorithmic stopping criterion for solving problems~\eqref{eq:general_optim} and \eqref{eq:dual_problem_screening}.
For any dual feasible vector $\theta$, we have $D(\hat \theta) \geq D(\theta)$. Moreover by weak duality, for any $\beta \in \dom P$ (primal feasible), we have $P(\beta) \geq D(\hat \theta)$. Whence,
\begin{equation*}
\hat \theta \in \{ \zeta \in \dom D:\, P(\beta) \geq D(\zeta) \geq D(\theta) \} \enspace.
\end{equation*}
Nevertheless, the screening test corresponding to such set may be hard to compute explicitly.
Various shapes have been considered in practice as a safe region $\mathcal{R}$.
Here we consider \emph{sphere regions} following the terminology introduced by \cite{ElGhaoui_Viallon_Rabbani12} \ie choosing a ball $\mathcal{R} = \ball(c,r)$ as a safe region.
In that case, by positive homogeneity of the support function, we have for any direction $d\in \bbR^{|g|}$, 
\begin{align*}
\sfunc{X_{g}^{\top} \ball(c,r)}(d) &=
\sfunc{X_{g}^{\top}c}(d) + \sfunc{X_{g}^{\top}\ball(0,r)}(d) \\
&\leq \sfunc{X_{g}^{\top}c}(d) + r \sup_{\norm{u}=1}\sfunc{X_{g}^{\top}u}(d) \enspace.
\end{align*}
For any group $g$ in $\mathcal{G}$, $\hat \beta_g = \beta^{\star}_g$ when for any $d$, the following \emph{sphere test} holds
\begin{equation}
\label{eq:sphere_test}
\sfunc{X_{g}^{\top}c}(d) + r \sup_{\norm{u}=1}\sfunc{X_{g}^{\top}u}(d) < \sfunc{\partial \Omega_g(\beta_{g}^{\star})}(d) \enspace.
\end{equation}
Note that when $\Omega_g$ is a norm the sphere test reduces to
\begin{equation*}
\Omega_{g}^{\circ}(X_{g}^{\top}c) + r \Omega_{g}^{\circ}(X_g) < 1 \enspace.
\end{equation*}
where $\Omega_{g}^{\circ}$ is the norm dual to $\Omega_g$.

For the non negativity constraint, the test reduces to
\begin{equation*}
-X_{j}^{\top} c + r|X_{j}| < 0 \enspace.
\end{equation*}

\subsection{Gap safe rules}
%


Under the smoothness assumption on the loss function, it was shown by \cite{Fercoq_Gramfort_Salmon15, Raj_Olbrich_Gartner_Scholkpof_Jaggi16, Ndiaye_Fercoq_Gramfort_Salmon17} that one can rely on the duality gap to construct a safe region, whenever a dual feasible point can be constructed.
Here, we recall this construction and show how to generalize it to construct a dual feasible point for a wider class of problems~\eqref{eq:general_optim}.

\begin{proposition}
\label{eq:gap_safe_sphere}
If the dual objective $D$ is $\mu_D$-strongly concave \wrt a norm $\norm{\cdot}$, for any primal/dual feasible vectors $(\beta, \theta) \in \dom P \times \dom D$, we have:
\begin{align}
 \normin{\hat \theta - \theta}^2 
 \label{ineq:gap_safe_radius}
 &\leq \frac{2}{\mu_D} \Gap(\beta, \theta) \enspace,
\end{align}
where $(\hat\beta, \hat \theta)$ is any primal/dual optimal solution.

Thus, the ball $\ball\left(\theta, \sqrt{\frac{2}{\mu_D} \Gap(\beta, \theta)} \right)$ is a safe region (called gap safe sphere).
\end{proposition}

We call \emph{gap safe rule} the sphere test in \Cref{eq:sphere_test} applied with the gap safe sphere in \Cref{eq:gap_safe_sphere}.
%

\paragraph{Construction of a dual feasible vector.}

To build a center for the safe sphere, we map a primal vector onto the dual space thanks to the gradient\footnote{any subgradient can be used when $f$ is not differentiable} mapping $\nabla f(\cdot)$.
However, the obtained dual vector is not necessarily feasible for the dual problem.
When the projection on the feasible set is hard, a generic procedure consists in performing a rescaling step so that it belongs to the dual set.
Precisely, we want to build $\theta \in \bbR^n$ such that
\begin{equation}\label{eq:dual_feasibility}
-\theta \in \dom f^{*} \text{ and } X^{\top}\theta \in \dom \Omega^{*} \enspace.
\end{equation}
Given a vector $\beta$ in $\bbR^p$, we build a dual point by rescaled gradient mapping  \ie
\begin{align}\label{eq:dual_feasible_point}
\theta &= \frac{-\nabla f(X\beta)}{\alpha} \\
\label{eq:generic_scaling}
\text{with}\quad \alpha &= \max\left(1 , \sfunc{\dom\Omega^{*}}^{\circ}\left(-X^{\top} \nabla f(X\beta)\right)\right) \enspace.
\end{align}
%
This choice is motivated by the primal-dual optimality link~\eqref{eq:fermat_rule}: $\hat \theta = -\nabla f(X \hat \beta)$.
Building a dual feasible vector by scaling is often used in the literature (see for instance \cite{ElGhaoui_Viallon_Rabbani12}) and reduces to residual rescaling, see \citep{Mairal,Massias_Gramfort_Salmon18} when $f$ is quadratic.


\begin{proposition}\label{prop:feasibility_dual_scaling}
If $f$ and $\Omega$ are bounded from below, then the dual vector $\theta$ in \Cref{eq:dual_feasible_point} satisfies the feasibility condition \eqref{eq:dual_feasibility}.
\end{proposition}

\begin{proof}
Since $\Omega$ is bounded from below then $\Omega^{*}(0) = - \inf_{z} \Omega(z) < +\infty$ which is equivalent to $\dom \Omega^*$ contains $0$. Since it is also closed and convex, we have $\sfunc{\dom \Omega^{*}}^{\circ}$ is positively homogeneous. Hence the vector $\theta$ in \Cref{eq:dual_feasible_point}
satisfies $\sfunc{\dom\Omega^{*}}^{\circ}(X^{\top} \theta) \leq 1$ which is equivalent to $X^{\top} \theta$ in $\dom\Omega^{*}$.
Moreover, by denoting $s={1}/{\alpha} \in [0,1]$, we have $-\theta = s \nabla f(X\beta) = s \nabla f(X\beta) + (1 - s) 0$. Since $\dom f^{*}$ is convex, it remains to show that it contains the vectors $\nabla f(X\beta)$ and $0$, thus it will necessarily contains $- \theta$ by convex combination.

Since $f$ is bounded from below, we have $f^{*}(0) = - \inf_{z} f(z) < +\infty$ which is equivalent to $0 \in \dom f^*$ . Moreover, the equality case in the Fenchel-Young inequality shows that $f(X\beta) + f^*(\nabla f(X\beta))= \langle \nabla f(X\beta), X\beta\rangle < +\infty$. Hence $f^*(\nabla f(X\beta))$ is also finite.
\end{proof}

\begin{remark}\label{rk:optimal_scaling}
The scaling constant $\alpha$ in \Cref{eq:generic_scaling} is equal to 1 if $\dom\Omega^{*}$ is unbounded.
\end{remark}

%% file: subfiles/complexity.tex
\section{How long does it take to identify the optimal active set with a screening rule?}
\label{sec:time_to_identify_the_optimal_active_set_with_a_screening_rule_}
We recall the notion of converging safe regions introduced in~\cite{Fercoq_Gramfort_Salmon15} that helps to reach exact active set identification in a finite number of steps.

\begin{definition}[Converging Safe Region]\label{def:converging_region}
Let $(\mathcal{R}_{k})_{k \in \bbN}$ be a sequence of compact convex sets containing the dual optimal solution $\hat \theta$. It is a converging sequence of safe regions if the diameters of the sets converge to zero. The associated safe screening rules are referred to as converging.
\end{definition}

The following proposition asserts that after a finite number of steps, the active set is exactly identified. Such a property is sometimes referred to as finite identification of the support~\cite{Liang_Fadili_Peyre17}.

\begin{proposition} \label{prop:active_set_identification}
Let $(\mathcal{R}_{k})_{k \in \bbN}$ be a sequence of compact convex set containing $\hat \theta$ for each $k$ in $\bbN$. If $(\mathcal{R}_{k})_k$ is converging in the sense of \Cref{def:converging_region},
then it exists an integer $k_0$ such that $\mathcal{A}_{\mathcal{R}_{k}} = \mathcal{A}$ for any $k \geq k_0$.
\end{proposition}

\begin{proof}
We proceed in two steps.

Firstly, we show that for any direction $d\in\bbR^{|g|}$, $\max_{\theta \in \mathcal{R}_k} \sfunc{X_{g}^{\top} \theta}(d)\rightarrow_k \sfunc{X_{g}^{\top} \hat \theta}(d)$. Indeed, for any $k \in \bbN$ and $\theta \in \mathcal{R}_{k}$ we have from the sublinearity and positive homogeneity of the support function, and since $\hat \theta$ in $\mathcal{R}_{k}$:
\begin{align*}
\sfunc{X_{g}^{\top} \hat \theta}(d) &\leq \max_{\theta \in \mathcal{R}_k} \sfunc{X_{g}^{\top} \theta}(d) \leq \sfunc{X_{g}^{\top}\hat \theta}(d) + \mathrm{diam}(\mathcal{R}_{k}) \sup_{\norm{u}=1}\sfunc{X_{g}^{\top}u}(d) \enspace,
\end{align*}
The conclusion follows from the fact that $\mathcal{R}_{k}$ is a converging sequence, since $\lim_{k \to \infty}\mathrm{diam}(\mathcal{R}_{k}) = 0$.

Secondly, we proceed by double inclusion.
First, remark that $\mathcal{A} = \mathcal{A}_{\mathcal{R}_{\infty}}$ where $\mathcal{R}_{\infty} :=\{\hat \theta\}$. So for all $k \in \bbN$, we have $\mathcal{A} \subseteq \mathcal{A}_{\mathcal{R}_{k}}$ since ($\mathcal{A}_{\mathcal{R}_{k}})_k$ are nested sequence of sets. Reciprocally, suppose that there exists a non active group $g \in \mathcal{G}$ \ie $\sfunc{X_{g}^{\top} \hat \theta}(d) < 1$ that remains in the active set $\mathcal{A}_{\mathcal{R}_{k}}$ for all iterations \ie  $\forall k \in \bbN, \,
\max_{\theta \in \mathcal{R}_k} \sfunc{X_{g}^{\top} \theta}(d) \geq 1$. Since $\lim_{k \rightarrow \infty} \max_{\theta \in \mathcal{R}_k} \sfunc{X_{g}^{\top} \theta}(d) = \sfunc{X_{g}^{\top} \hat \theta}(d)$, we obtain $\sfunc{X_{g}^{\top} \hat \theta}(d) \geq 1$ by passing to the limit. Hence, by contradiction, there exits an integer $k_0 \in \bbN$ such that $[p] \backslash \mathcal{A} \subseteq \mathcal{A}_{\mathcal{R}_{k}}^{c}$ for all $k \geq k_0$.
\end{proof}

One can note that the rate of identification of the active set is strongly related to the rate at which the sequence of diameters $\mathrm{diam}(\mathcal{R}_{k})$ goes to zero during the optimization process. We quantify this in the next section.

\subsection{Complexity of safe active set identification}

Dynamic safe screening rules have practical benefits since they increase the number of screened out variables as the algorithm proceeds.
The next proposition states that if one relies on a primal converging algorithm, then the dual sequence we propose is also converging.
It only requires uniqueness of $X\hat \beta$ (not that of $\hat \beta$).
In the following, $\beta_k$ is the current estimate of a primal solution $\hat \beta$ and $\theta_k = -\nabla f(X\beta_k)/ \alpha_k$, with $\alpha_k = \max(1 , \sfunc{\dom\Omega^{*}}^{\circ}(X^{\top} \nabla f(X\beta_k)))$, be the current estimate of the dual solution $\hat \theta$.

\begin{lemma}
\label{lm:dual_convergence}
It holds $\lim_{k \to \infty} X\beta_k = X\hat \beta$ implies
$\lim_{k \to \infty}\theta_k=\hat \theta$.
\end{lemma}

\begin{proof}
Let $\alpha_k = \max(1, \sfunc{\dom \Omega^*}^{\circ}(X^\top \nabla f(X\beta_k)))$,
we have:
\begin{align*}
\normin{\theta_k-\hat \theta}_2 & = \norm{\nabla f(X \hat \beta) - \frac{\nabla f(X \beta_k)}{\alpha_k}}_2\\
& \leq \left|1 -\frac{1}{\alpha_k}\right| \normin{\nabla f(X \beta_k)}_2 + \norm{\nabla f(X \hat \beta)-\nabla f(X \beta_k)}_2.
\end{align*}
If $X\beta_k \rightarrow X\hat \beta$ holds, then $\alpha_k \rightarrow
\max(1, \sfunc{\dom \Omega^*}^{\circ}(X^\top \nabla f(X\hat \beta)))=
\max(1,  \sfunc{\dom \Omega^*}^{\circ}(X^\top \hat \theta)) = 1$, since
$\nabla f(X\hat \beta) = - \hat \theta$ thanks to the optimality condition and the feasibility of $\hat \theta$:
$\sfunc{\dom \Omega^*}^{\circ}(X^\top \hat \theta)~\leq~1$. Hence the right hand side of the previous inequality converges to zero, and the conclusion holds.
\end{proof}


From \Cref{lm:dual_convergence} and by strong duality, the sequence of radius $r_k = (2 \Gap (\beta_k, \theta_k)/ \mu_D)^{1/2}$ converges to $0$ as $k$ goes to $\infty$.
Hence, the sequence of safe balls $\ball(\theta_k, r_k)$ converges to $\{\hat \theta\}$. Whence, we deduce the following property.
\begin{proposition}\label{prop:gap_safe_is_converging}
The Gap Safe rules produce converging safe regions.
\end{proposition}

The results in \Cref{prop:active_set_identification} and \Cref{prop:gap_safe_is_converging} ensure that screening rules, applied iteratively with the duality gap based safe region, will identify the active set after a finite number of iterations.

For any safe ball $\ball(c_k, r_k)$, we have the inequalities
\begin{align*}
\sfunc{X_{g}^{\top} \hat \theta}(d) &\leq \max_{\theta \in \ball(c_k, r_k)} \sfunc{X_{g}^{\top} \theta}(d)
\leq \sfunc{X_{g}^{\top}\hat \theta}(d) + 2 r_k \sup_{\norm{u}=1}\sfunc{X_{g}^{\top}u}(d) \enspace,
\end{align*}
and the identification of the active set occurs when for all group $g$ in $\mathcal{Z}$ and any direction $d$, we have $\sfunc{X_{g}^{\top} \hat \theta}(d) < \sfunc{\partial \Omega_g(\beta_{g}^{\star})}(d)$.
The latter holds as soon as\footnote{where we implicitly avoid the trivial case where there exists some group $g$ such that $\sup_{\norm{u}=1}\sfunc{X_{g}^{\top}u}(d)=0$.}
\begin{align*}
r_k < \frac{1}{2} \min_{\underset{d \in \bd\ball(0,1)}{g \in \mathcal{Z}}}
\frac{\sfunc{\partial \Omega_g(\beta_{g}^{\star})}(d) - \sfunc{X_{g}^{\top}\hat \theta}(d)}{\sup_{\norm{u}=1}\sfunc{X_{g}^{\top}u}(d)} =: \delta_{\mathcal{Z}} \enspace.
\end{align*}
Whence, the identification of the active set using a safe ball of radius $r_k$ occurs after $k_0$ iterations where
\begin{align}\label{eq:genereic_identification_constant}
k_0 := \inf\{k \in \bbN:\, r_k < \delta_{\mathcal{Z}}\} \enspace.
\end{align}

\begin{remark}[Non-degeneracy condition]
By definition, the set $\mathcal{Z}$ is empty if $\delta_{\mathcal{Z}}$ is equal to zero. Thus all the complexity bounds are equal to infinity and then $\delta_{\mathcal{Z}} > 0$ is a necessary non-degeneracy condition to ensure finite identifications of the active set.
\end{remark}

\subsection{Duality gap certificates}

Recently, a complexity analysis of the convergence of the duality gap, used as an optimality certificate as been proposed \cite{Dunner_Forte_Takac_Jaggi16}. This analysis is important for deriving the complexity of active set identification that depends only on the rate of convergence of the algorithm. The next lemma adapts the proposed analysis to that take dual rescaling into account.

\begin{lemma}\label{lm:dunner_lemma}
Let $f$ be $\nu_f$-smooth and $\Omega$ be $\mu_{\Omega}$-strongly convex ($\mu_{\Omega} = 0$ is allowed when $\Omega^*$ is subdifferentiable on its domain \footnote{Note that in particular, the rescaled gradient mapping allows $\mu_{\Omega} = 0$ without restricting $\Omega$ to have a bounded support.}). Since the dual vector $\theta$ in \eqref{eq:dual_feasible_point} is feasible, we can choose $u \in \partial\Omega^*(X^\top \theta) \neq \emptyset$. For all $s$ in $[0,1]$, it holds
\begin{align}\label{eq:subopt_lower_bound}
P(\beta) - P(\hat \beta) \geq &s (\Gap(\beta, \theta) + \Delta(\alpha)) +
 s^2 \left[ \frac{(1-s)\mu_{\Omega}}{s}\norm{\beta - u}^2 - \frac{\nu_f}{2}\norm{X(u-\beta)}^2 \right]
\end{align}
with $\Delta(\alpha) = f^*(\nabla f(X\beta)) - f^*(-\theta) + (\alpha - 1) \langle \theta, Xu\rangle$ and the scaling $\alpha$ is defined in \Cref{eq:generic_scaling}.
\end{lemma}

\begin{proof}

By optimality of $\hat \beta$, for any $\beta$ and $u$ in $\dom P$, we have:
\begin{align}\label{eq:bound_1}
P(\beta) - P(\hat \beta) &\geq P(\beta) - P(\beta + s(u - \beta)) \nonumber\\
&= [\Omega(\beta) - \Omega(\beta + s (u - \beta))] + [f(X\beta) - f(X(\beta + s(u - \beta)))] \enspace.
\end{align}

By strong convexity of $\Omega$, we have:
\begin{align}\label{eq:SC_of_Omega}
\Omega(\beta) - \Omega(\beta + s (u - \beta)) \geq s (\Omega(\beta) - \Omega(u)) + \frac{s(1-s) \mu_{\Omega}}{2} \norm{u - \beta}^2 \enspace.
\end{align}

From the smoothness of $f$, we have:
\begin{align}\label{eq:SM_of_f}
f(X\beta) - f(X\beta + sX(u - \beta)) \geq s \langle \nabla f(X\beta), X(u - \beta) \rangle - \frac{s^2 \nu_f}{2}\norm{X(u - \beta)}^2 \enspace.
\end{align}

Then, plugging \Cref{eq:SC_of_Omega} and \Cref{eq:SM_of_f} to \Cref{eq:bound_1}, yields:
\begin{align*}
P(\beta) - P(\hat \beta) \geq s \Gamma + \frac{s^2}{2} \left[ \frac{(1-s)\mu_{\Omega}}{s}\norm{u - \beta}^2 - \nu_f\norm{X(u-\beta)}^2 \right]\enspace,
\end{align*}
where $\Gamma = \Omega(\beta) - \Omega(u) - \langle \nabla f(X\beta), X(u - \beta) \rangle$.

The choice of the scaling $\alpha$ in \cref{eq:generic_scaling}, we have $X^\top\theta \in \dom \Omega^*$ which implies that $\partial \Omega^*(X^\top \theta)$ is non empty.
Thus we can choose $u \in \partial \Omega^*(X^\top \theta)$ which ensure that $u \in \dom \Omega$.
Also, $f$ is smooth if and only if $f^*$ is strongly convex which implies that $\dom f$ is the whole space.
Thus $Xu \in \dom f$.
Whence $u \in \dom \Omega$ and $Xu \in \dom f$ implies $u \in \dom P$.
Let $\beta \in \dom P$, the for any $s \in [0, 1]$ on can check that $\beta + s(u - \beta) \in \dom P$.

For $u \in \partial\Omega^*(X^\top \theta)$, the equality case in the Fenchel-Young inequality reads:
$$
\Omega(u) = \langle u, X^\top \theta \rangle - \Omega^*(X^\top \theta) \enspace.
$$
Whence,
\begin{align*}
\Gamma &= \Omega(\beta) + \Omega^*(X^\top \theta) - \langle u, X^\top \theta \rangle - \langle \nabla f(X\beta), X(u - \beta) \rangle \\
&= \Gap(\beta, \theta) - f(X\beta) - f^*(-\theta) - \langle u, X^\top \theta \rangle - \langle \nabla f(X\beta), X(u - \beta) \rangle \enspace.
%
\end{align*}
From the equality case in the Fenchel-Young inequality, we have
$f(X\beta) = \langle \nabla f(X\beta), X\beta \rangle - f^*(\nabla f(X\beta))$.
Thanks to the last display and to the definition of $\theta$, we have $-f(X\beta) - f^*(-\theta) - \langle u, X^\top \theta \rangle - \langle \nabla f(X\beta), X(u - \beta) \rangle = \Delta(\alpha)$, hence the result.

\end{proof}

Let us denote the sub-optimality gap
\begin{align}\label{eq:sub-opt-gap}
\mathcal{E}_k = P(\beta_k) - P(\hat \beta), \text{ for } k \in \bbN \enspace.
\end{align}

\paragraph{Cases where $\mu_{\Omega} > 0$.}
In such a case, $\dom\Omega^*$ is the whole dual space and we can choose $\alpha = 1$ (see \Cref{rk:optimal_scaling}) whence $\Delta(\alpha)=0$.
Now choosing $s = \frac{\mu_{\Omega}}{\sigma_X\nu_{f} + \mu_{\Omega}}$ where $\sigma_X$ is the spectral norm of the design matrix $X$ (see also \cite{Dunner_Forte_Takac_Jaggi16}), then the last term in \Cref{eq:subopt_lower_bound} vanishes. Thus,
\begin{align*}
\frac{\mu_{\Omega}}{\sigma_X\nu_{f} + \mu_{\Omega}} \Gap(\beta_{k}, \theta_{k}) &\leq \mathcal{E}_{k} \leq \Gap(\beta_{k}, \theta_{k})
\enspace.
\end{align*}
This guarantees that the duality gap converges at the same rates as the sub-optimality gap.
Along with \Cref{eq:genereic_identification_constant}, we obtain the following proposition.
%
%
\begin{proposition}\label{prop:linear_complexity}
For $\mu_{\Omega}>0$ and any linearly converging primal algorithm \ie with $\mathrm{Rate}(k) = \exp(-\kappa k)$, the active set will be identified after at most $k_0$ iterations where
\begin{align*}
k_0 \leq \frac{1}{\kappa}\log\left(\frac{C_{f,\Omega, X}}{\delta_{\mathcal{Z}}^{2}} \frac{2}{\mu_D} \mathcal{E}_0\right) \enspace,
\end{align*}
for some $\kappa$ in $(0, 1]$ and the constant $C_{f,\Omega, X} := \frac{\sigma_X\nu_{f} + \mu_{\Omega}}{\mu_{\Omega}}$ depends only on the conditioning of the design matrix $X$ and on the regularity of $f$ and $\Omega$.
\end{proposition}

\paragraph{Case where $\mu_{\Omega} = 0$.} One possibility, here, is to modify $\Omega$ by adding a small strongly convex term (\eg smoothing).
Then, the previous result still holds for the modified problem.
However, this will slightly modify the iterates of the algorithm.
Otherwise, one can assume that $\Omega$ has a bounded support \ie $\dom \Omega$ is included in a ball of radius $L$.
In such a case, $\Omega^*$ is finite everywhere and we can still choose $\alpha = 1$ whence $\Delta(\alpha) = 0$ while having
\begin{align}\label{eq:loose_bound}
\norm{X(u_{k} - \beta_{k})} \leq 2 \sigma_{X} L \enspace.
\end{align}
Plugging it into \Cref{lm:dunner_lemma}, we obtain
\begin{align*}
\Gap(\beta_{k}, \theta_{k}) &\leq \frac{1}{s} \mathcal{E}_{k} + 2\nu_f \sigma_{X}^{2} L^2 s \enspace.
\end{align*}
Minimizing the upper bound in $s$ onto $(0, 1]$, we have
\begin{align*}
\Gap(\beta_{k}, \theta_{k}) \leq \sqrt{8 \nu_f \sigma_{X}^{2} L^2 \mathcal{E}_{k}} \enspace.
\end{align*}
When the optimization algorithm converges linearly, the complexity in Prop.~\ref{prop:linear_complexity} is preserved up to some constants because the logarithmic term is not affected by the square-root.
But, it leads to a suboptimal bound in the sub-linear regime,
\begin{proposition}\label{eq:gap_rate_sublinear_suboptimal}
For $\mu_{\Omega}=0$ and any sub-linearly primal convergent algorithm \ie with $\mathrm{Rate}(k) = C/k^\gamma$ where $\gamma>0$, the active set will be identified after at most $k_0$ iterations where
\begin{align*}
k_0 \leq \left( \frac{8 \nu_f \sigma_{X}^{2}L^2C}{(\mu_{D} \delta_{\mathcal{Z}}^{2})^2} \right)^{\tfrac{1}{\gamma}} \enspace.
\end{align*}
\end{proposition}
%
To exactly match the rate of the algorithm (\ie to remove the squared term \Cref{eq:gap_rate_sublinear_suboptimal}) we propose to additionally assume Lipschitz continuity of the sub-differential $\partial \Omega^*$ and a quadratic error bound on the objective function $P$.
More precisely, we suppose that there exists some constants $L_*$ and $\gamma_P > 0$  such that for a selection of $u$ in $\partial \Omega^*(\zeta)$ and $\hat u$ in $\partial \Omega^*(\hat \zeta)$, we have\footnote{This condition is required only for $\hat \zeta = -X^\top \nabla f(X\hat \beta)$ and when $\zeta = \hat \zeta$, the choice restricts to $u = \hat u = \hat \beta$, which can be ensured by selecting $u$ as the projection of $\beta$ onto $\partial \Omega(X^\top\theta)$.}:
\begin{align}
\norm{u - \hat u} &\leq L_* \normin{\zeta - \hat \zeta} \label{eq:lipschitz_subdifferential} \\
\frac{\gamma_P}{2} \normin{\beta - \hat \beta}^2 &\leq P(\beta) - P(\hat \beta) \enspace. \label{eq:error_bound}
\end{align}

The reason is that $u_{k}$ is expected to converge to $\hat\beta$ and so the bound in \Cref{eq:loose_bound} may be too crude. In a sub-linear regime, the following lemma shows that assumptions~\eqref{eq:lipschitz_subdifferential} and \eqref{eq:error_bound} are sufficient conditions to improve the previous analysis in \cite{Dunner_Forte_Takac_Jaggi16}.

\begin{remark}
The quadratic error bound condition in \Cref{eq:error_bound} was proven to be satisfied for a large class of optimization problem. One can refer to \cite{Bolte_Nguyen_Peypouquet_Suter16} where it was used to analyze the complexity of first order optimization methods. Similar Lipschitz continuity assumptions on the subdifferential in \Cref{eq:lipschitz_subdifferential} were made in \cite{Jourani_Thibault_Zagrodny12}, see also \citep[Chapter 9.E]{Rockafellar_Wets09}. However, it is not straightforward to explicitly compute these constants for practical applications.
\end{remark}

\begin{lemma}\label{eq:improved_gap_bound}
Under assumptions~\eqref{eq:lipschitz_subdifferential} and \eqref{eq:error_bound}, for any integer $k$, it holds
\begin{align*}
\Gap(\beta_{k}, \theta_{k}) &\leq \sqrt{2\nu_f C'_{f, \Omega, X}} \; \mathcal{E}_{k} \enspace,
\end{align*}
where $C'_{f, \Omega, X} = \frac{4 \sigma_{X}^{2}(L_{*}^{2} \sigma_{X}^{2}\nu_{f}^{2} + 1)}{\gamma_P}$ is non negative and finite.
\end{lemma}

\begin{proof}
First note that
\begin{align*}
\norm{X(u_{k} - \beta_{k})}^2 \leq 2 \sigma_{X}^{2}(\normin{u_{k} - \hat\beta}^2 + \normin{\beta_{k} - \hat\beta}^2) \enspace.
\end{align*}
When $\alpha=1$, we have $u_{k} \in \partial \Omega^*(-X^\top \nabla f(\beta_{k}))$.
Moreover, $\hat \beta \in \partial \Omega^*(-X^\top \nabla f(\hat\beta))$ and we have:
\begin{align*}
\normin{u_{k} - \hat\beta} &\leq L_* \normin{X^\top \nabla f(X\beta_{k}) - X^\top \nabla f(X\hat\beta)} \\
&\leq L_* \sigma_X \nu_f\normin{\beta_{k} - \hat\beta} \\
&\leq L_* \sigma_X \nu_f \sqrt{\frac{2}{\gamma_P} \mathcal{E}_{k}} \enspace,
\end{align*}
where the first inequality results from the assumption~\eqref{eq:lipschitz_subdifferential}, the second from the smoothness of $f$ ($f$ is $\nu_f$ smooth so the gradient is $\nu_f$ Lipschitz), and the third from the quadratic error bound~\Cref{eq:error_bound}.

Thus, for $C'_{f, \Omega, X} = \frac{4 \sigma_{X}^{2}(L_{*}^{2} \sigma_{X}^{2}\nu_{f}^{2} + 1)}{\gamma_P}$, we have
\begin{align*}
\norm{X(u_{k} - \beta_{k})}^2 \leq C'_{f, \Omega, X} \mathcal{E}_{k} \enspace.
\end{align*}
Plugging it into \Cref{lm:dunner_lemma}, we obtain
\begin{align*}
\mathcal{E}_{k} \geq P(\beta_k) - P(\hat \beta) \geq &s \Gap(\beta, \theta) -
s^2 \frac{\nu_f C'_{f, \Omega, X}}{2} \mathcal{E}_{k} \enspace.
\end{align*}
Whence
\begin{align*}
\Gap(\beta_{k}, \theta_{k}) &\leq \left(\frac{1}{s} + s \frac{\nu_f C'_{f, \Omega, X}}{2} \right) \mathcal{E}_{k} \enspace.
\end{align*}
Minimizing the upper bound in $s$ onto $(0, 1]$, we obtain the result.
\end{proof}


%
Then, in a sub-linear regime, we recover the exact rate.
\begin{proposition}
For $\mu_{\Omega}=0$, under assumptions~\eqref{eq:lipschitz_subdifferential} \eqref{eq:error_bound} and any sub-linearly primal convergent algorithm \ie with $\mathrm{Rate}(k) = C/R^{\gamma}$ where $\gamma>0$, the active set will be identified after at most $k_0$ iterations where
\begin{align*}
k_0 \leq \left( \frac{\sqrt{8 \nu_{f}C'_{f, \Omega, X}} C}{\mu_{D} \delta_{\mathcal{Z}}^{2}} \right)^{\tfrac{1}{\gamma}} \enspace.
\end{align*}
\end{proposition}

Finally, when the domain of $\Omega$ is not bounded, the algorithm can be equipped with a modified duality gap which enforces the bounded domain assumption (this is known as the \emph{Lipschitzing Trick} \cite{Dunner_Forte_Takac_Jaggi16} in the litterature).
Then, the previous result still holds without modifying the iterates of the algorithm.


\paragraph{Related works.}
To our knowledge, this paper is the first one to discuss the complexity of active set identification with screening rules.
Our results match the existing results on active set identification in \cite{Liang_Fadili_Peyre17, Nutini_Laradji_Schmidt17, Sun_Jeong_Nutini_Schmidt19} for proximal algorithms.
Interestingly, our result uniformly holds for any converging algorithm not only proximal methods and illustrates the benefits obtained as screening rules explicitly and definitely eliminate non-active variables along the algorithmic progress.

%% file: subfiles/acceleration_strategies.tex

\section{Acceleration Strategies}
\label{sec:Acceleration_Strategies}

We discuss some practical methods for efficiently using screening rules to speed up optimization processes for solving \Cref{eq:general_optim} and show how some popular previous acceleration heuristics such as \emph{strong rules} \cite{Tibshirani_Bien_Friedman_Hastie_Simon_Tibshirani12} or recent \emph{working sets} \cite{Johnson_Guestrin15, Massias_Gramfort_Salmon18} can be extended in our framework.

\paragraph{Static (Pre-processing).}
A natural strategy is to set, once for all, a gap safe radius using some initial fixed vectors $\theta = \theta_0$ and $\beta = \beta_0$. The resulting static safe region $\ball\left(\theta_0, \sqrt{ \frac{2}{\mu_D} \Gap(\beta_0, \theta_0) } \right)$ is used in \Cref{eq:sphere_test}.
Such a strategy is only efficient when $(\beta_0, \theta_0)$ are good enough estimate of the optimal solutions, and have limited scope in practice.

\paragraph{Dynamic.} One could rather use the information gained during an optimization process to obtain a smaller safe region therefore a greater elimination of inactive variables. Whence, we consider
$
\ball\left(\theta_k, \sqrt{ \frac{2}{\mu_D} \Gap(\beta_k, \theta_k } \right).
$
Dynamic safe region was initially suggested in \cite{Bonnefoy_Emiya_Ralaivola_Gribonval14} and further used in the duality gap based region in  \cite{Fercoq_Gramfort_Salmon15, Shibagaki_Karasuyama_Hatano_Takeuchi16, Ndiaye_Fercoq_Gramfort_Salmon17,LeMorvan_Vert18}.

\paragraph{Sequential (Homotopy Continuation).}

Sequential screening is motivated by the intuition that, often, the duality gap grows continuously \wrt to the regularization parameter \cite{Giesen_Laue_Mueller_Swiercy12, Ndiaye_Le_Fercoq_Salmon_Takeuchi2019}.
%
It basically states that when $\lambda$ close to $\lambda_t$, the duality gap $\Gap_{\lambda}(\beta^{(\lambda_t)}, \theta^{(\lambda_t)})$ tends to $\Gap_{\lambda_t}(\beta^{(\lambda_t)}, \theta^{(\lambda_t)})$.
As a by product, given a sufficiently fine grid of parameter $(\lambda_t)_{t \in [T]}$, the sequential screenings based on balls $\ball\left(\theta^{(\lambda_t)}, \sqrt{ \frac{2}{\mu_D} \Gap_{\lambda_{t-1}}(\beta^{(\lambda_t)}, \theta^{(\lambda_t)})} \right)$, will be small enough to efficiently remove non active variables.

\paragraph{Active Warm Start (aka strong rules).}
This method was introduced in \cite{Tibshirani_Bien_Friedman_Hastie_Simon_Tibshirani12} as a heuristic relaxation of the safe rules to discard features more aggressively in $\ell_1$ regularized optimization problem. We generalize it into our framework.
Let $F(\beta) = f(X\beta)$. We havefor any $d \in \bbR^{|g|}$
\begin{align*}
\sfunc{X_{g}^{\top}\ttheta{\lambda}}(d) &= \sfunc{X_{g}^{\top}\ttheta{\lambda'}}(d) + \sfunc{\{X_{g}^{\top}\ttheta{\lambda} - X_{g}^{\top}\ttheta{\lambda'}\}}(d) \\
&= \sfunc{X_{g}^{\top}\ttheta{\lambda'}}(d) + \sfunc{\{\nabla_g F(\tbeta{\lambda'}) - \nabla_g F(\tbeta{\lambda})\}}(d) \enspace.
\end{align*}
If $\nabla F$ is group-wise non-expansive along the regularization path \ie
$\normin{\nabla_g F(\tbeta{\lambda'}) - \nabla_g F(\tbeta{\lambda})} \leq |\lambda' - \lambda|$,
the screening
 holds whenever the
\emph{(generalized) strong rule} holds:
\begin{align*}
\sfunc{X_{g}^{\top}\ttheta{\lambda'}}(d) + |\lambda' - \lambda| < \sfunc{\partial \Omega_g(\beta_{g}^{\star})}(d) \enspace.
\end{align*}
The strong rules are un-safe because the non-expansiveness condition on $\nabla F$ is usually not satisfied without stronger assumptions on the design matrix $X$ (\eg $X$ has full column rank and $(X^\top X)^{-1}$ is diagonally dominant). Moreover, the \emph{exact} solution $\ttheta{\lambda'}$ is usually not available.\\
As a simpler rule, specially when the previous regularity condition cannot be verified, we rather suggest to use the previous active set
\begin{align*}
    \sfunc{X_{g}^{\top}\ttheta{\lambda'}}(d) < \sfunc{\partial \Omega_g(\beta_{g}^{\star})}(d) \enspace.
\end{align*}
%

%
%
The rational behind these heuristics is that, often, the active set is stable along the regularization path, a crucial argument used to build the \texttt{Lars} algorithm \cite{Efron_Hastie_Johnstone_Tibshirani04} and variants \cite{Osborne_Presnell_Turlach00b}.

\paragraph{Aggressive Active Warm Start.}

The gap safe screening rule relies on an upper estimates the suboptimal gap by the duality gap $\Gap(\beta_k, \theta_k)$. This can be conservative for the screening rules since no false elimination is allowed. Here we suggest a new heuristic in order to remove more variables at an early stage of an optimization process. At any iteration $k$, use $\mathcal{E}_{k}^{\approx} = |P(\beta_{k - s}) - P(\beta_k)|$ as an unsafe estimate of the suboptimal gap.
This will eliminate more variables depending on the choice of the delaying parameter $s$. In practice, we delay $\beta_k$ and $\beta_{k-s}$ with $10$ epochs for instance for the Lasso case, when using coordinate descent as a solver. To avoid a severe underestimation, one can instead use $(1 - \eta) \mathcal{E}_{k}^{\approx} + \eta \Gap(\beta_k, \theta_k)$. We set a default value $\eta=10^{-3}$.\\
See the numerical illustrations in \Cref{fig:rate_aggressive} and appendix.
%
\begin{remark}
Since these rules are unsafe \ie they can wrongly remove some variables, they must be accompanied with a post-precessing step. For instance by adding back the variables that violates the KKT conditions. We rather suggest to use the solution obtained in these steps as a warm start for the dynamic safe rules with a converging algorithm. In this way, a low computational complexity can be maintained when passing over the entire problem with a better initialization of gap safe screening rules.
\end{remark}

\begin{figure}[!t]
  \centering
  \subfigure[Convergence rate.]{\includegraphics[width=0.48\linewidth, keepaspectratio]{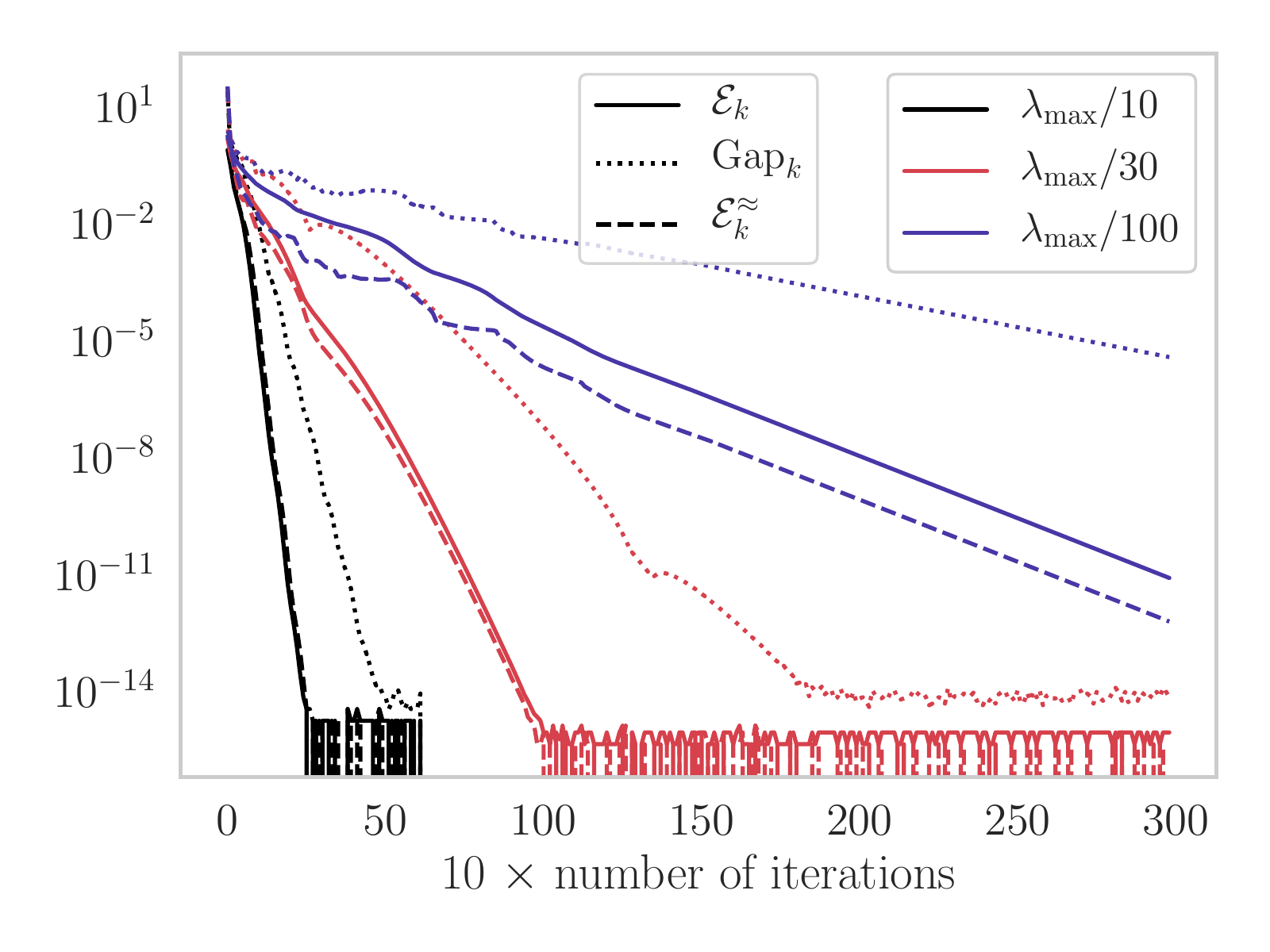}}
  %
  %
  \subfigure[Times for $\lambda=\lambda_{\max}/10$.]{\includegraphics[width=0.48\linewidth, keepaspectratio]{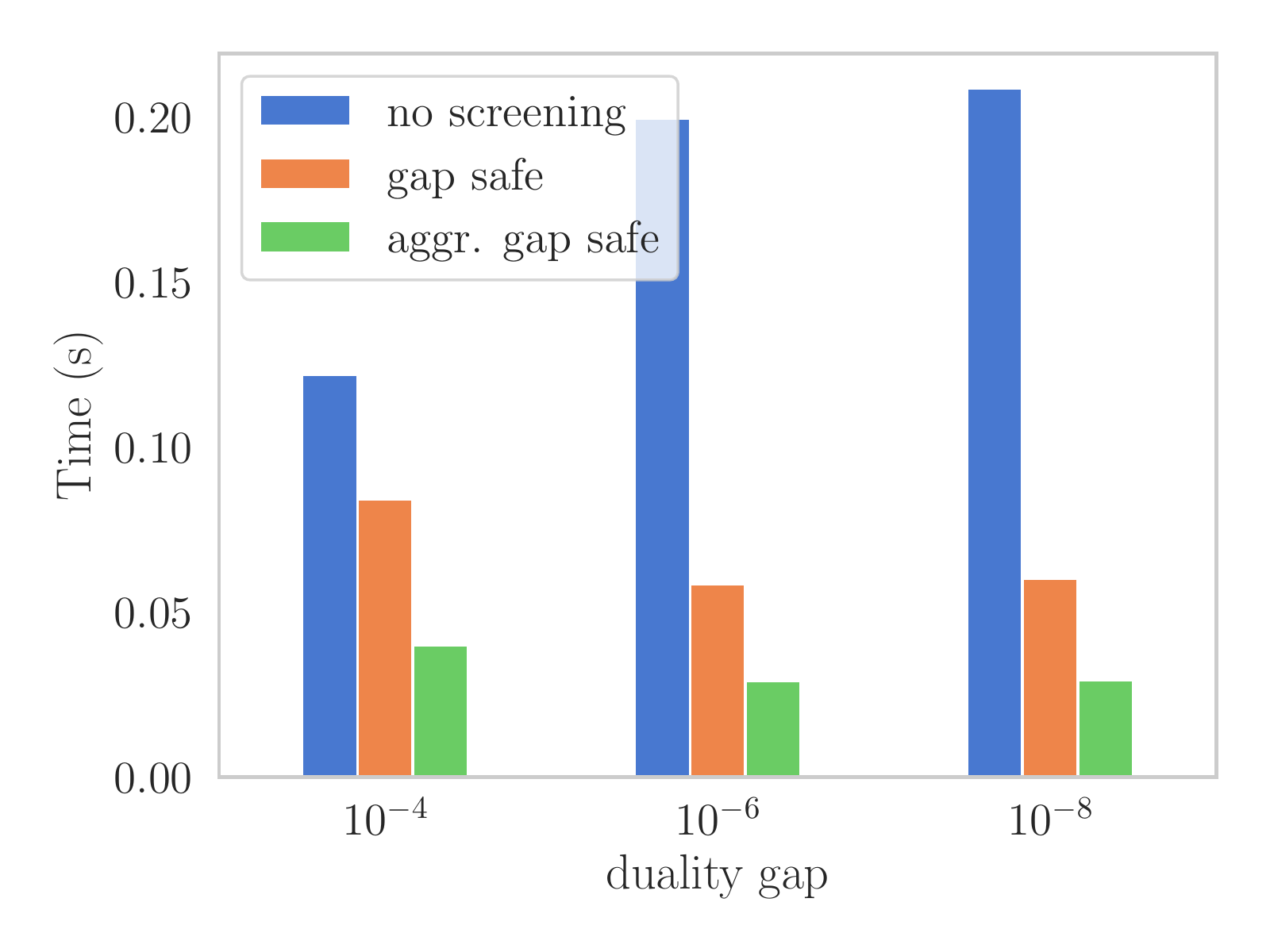}}
  %
  %
  \subfigure[Times for $\lambda=\lambda_{\max}/30$.]{\includegraphics[width=0.48\linewidth, keepaspectratio]{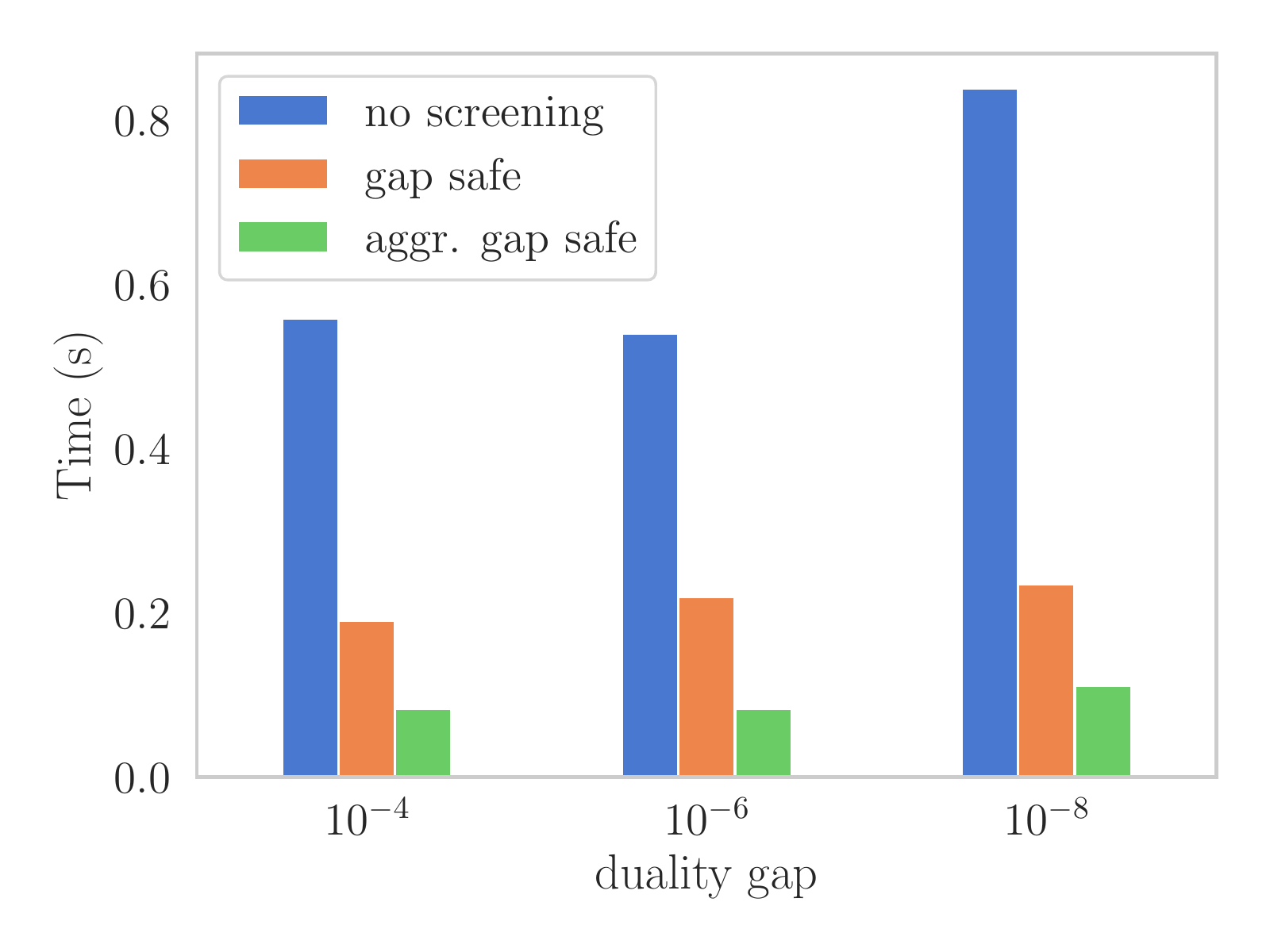}}
  %
  %
  \subfigure[Times for $\lambda=\lambda_{\max}/100$.]{\includegraphics[width=0.48\linewidth, keepaspectratio]{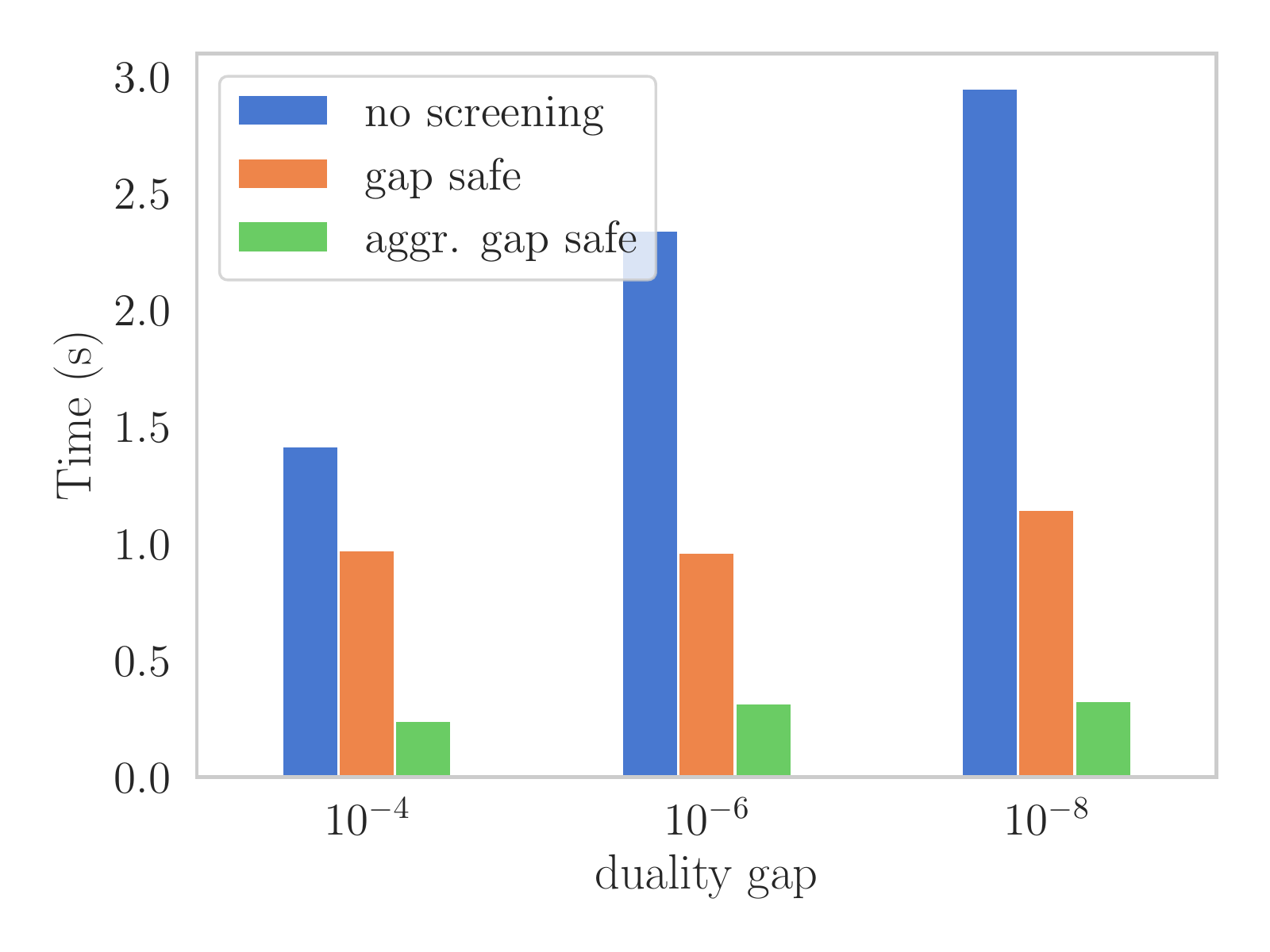}}
  \caption{Illustrations on Lasso using (cyclic) coordinate descent on Leukemia dataset ($n=72$ observations and $p=7129$ features. Here $\lambda_{\max} = \normin{X^\top y}_{\infty}$ is the smallest $\lambda$ such that $\hat\beta=0$ is a primal optimal solution. \label{fig:rate_aggressive}}
\end{figure}

\paragraph{Working Sets.}
Following the suggestions made in \cite{Johnson_Guestrin15, Massias_Vaiter_Gramfort_Salmon19}, one can consider, for any group $g$ in $\mathcal{G}$
\begin{align}
d_g(\theta) = \frac{\sfunc{\partial \Omega_g(\beta_{g}^{\star})}(d) - \sfunc{X_{g}^{\top} \theta}(d)}{\sup_{\norm{u}=1}\sfunc{X_{g}^{\top}u}(d)} \enspace,
\end{align}
as a measure of the importance of feature $X_g$. Thus, one can design a working set \ie a set of group $g$ in which to restrict the optimization problem, by selecting the groups that have a higher value $d_g(\theta)$.
These methods fit naturally in our framework.

%% file: subfiles/discussions.tex

\section{Numerical Experiments}

We consider simple examples to illustrate the performance of different acceleration strategies with screening rules on Lasso problem with real datasets.
We use a cyclic coordinate descent solver \footnote{The implementation is available at \url{https://github.com/EugeneNdiaye/Gap_Safe_Rules}} as a shared standard algorithm for all methods.
All methods are stopped when the duality gap reaches a prescribed tolerance $\epsilon \norm{y}^2$ where $\epsilon$ is set to $10^{-4}$, $10^{-6}$ or $10^{-8}$.
For readability, the execution times of the algorithms are normalized with respect to the running time of coordinate descent with the gap safe screening rule baseline as done in \cite{Ndiaye_Fercoq_Gramfort_Salmon17}. Evaluations of the performance of safes rules for other problems such as logistic regression, Sparse-Group Lasso, SVM etc are available in the literature \eg \cite{Ndiaye_Fercoq_Gramfort_Salmon15, Ndiaye_Fercoq_Gramfort_Salmon16, Shibagaki_Karasuyama_Hatano_Takeuchi16}.

Although safe rules can save a significant amount of computational time, they should be conservative so as not to wrong eliminate relevant variables.
In our numerical experiments, we observe that this constraint can limit their efficiency. By reducing this safety constraint, one can greatly improve their efficiency by combining them with a simple heuristic like the one introduced in \Cref{sec:Acceleration_Strategies}.

\begin{figure}[H]
\includegraphics[width=\linewidth]{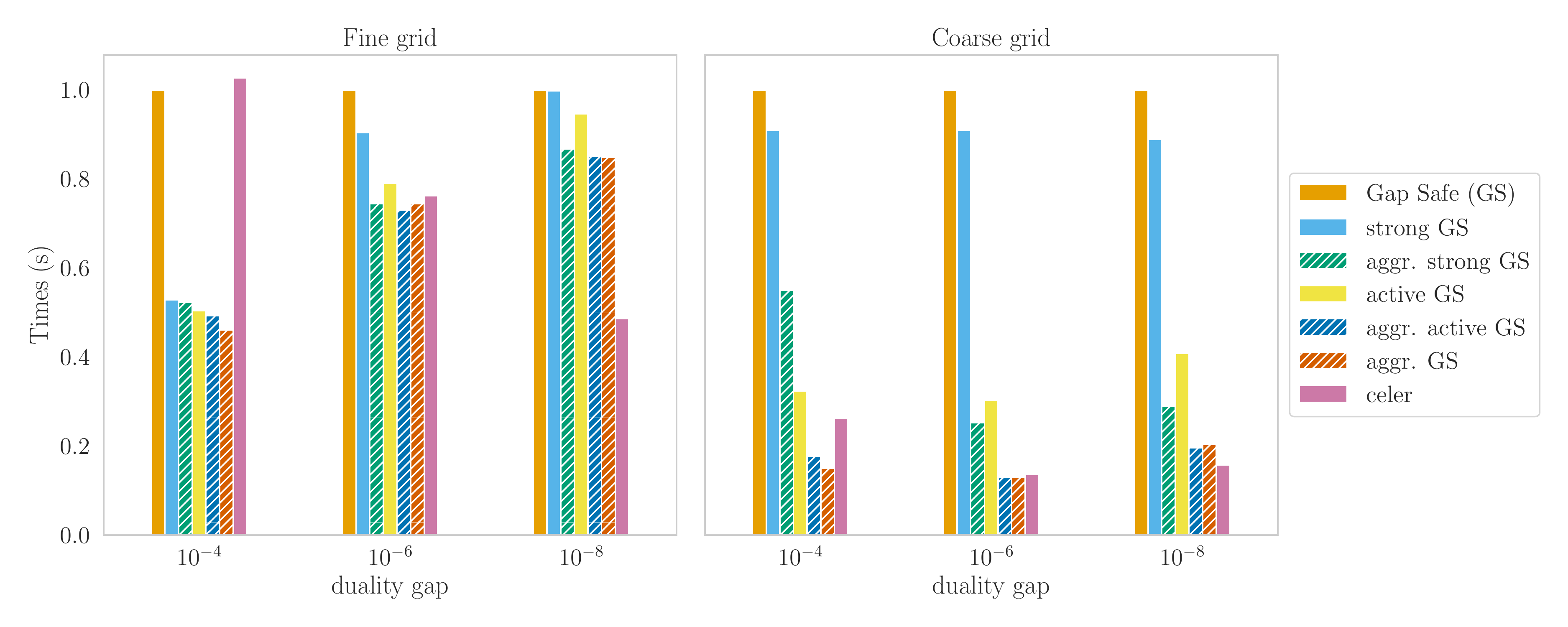}\label{subfig:lasso_leukemia}
\caption{Lasso on the \texttt{Leukemia} (dense data with $n=72$ observations and $p=7129$ features). Computation times needed to solve the Lasso regression path to desired accuracy for a grid of $\lambda$ from $\lambda_{\max} = \norm{X^\top y}_{\infty}$ to $\lambda_{\max} / 100$. The size of the dense grid (resp. sparse grid) is $100$ (resp. 10).
}\label{fig:lasso_leukemia}
\end{figure}

\begin{figure}[H]
\includegraphics[width=\linewidth]{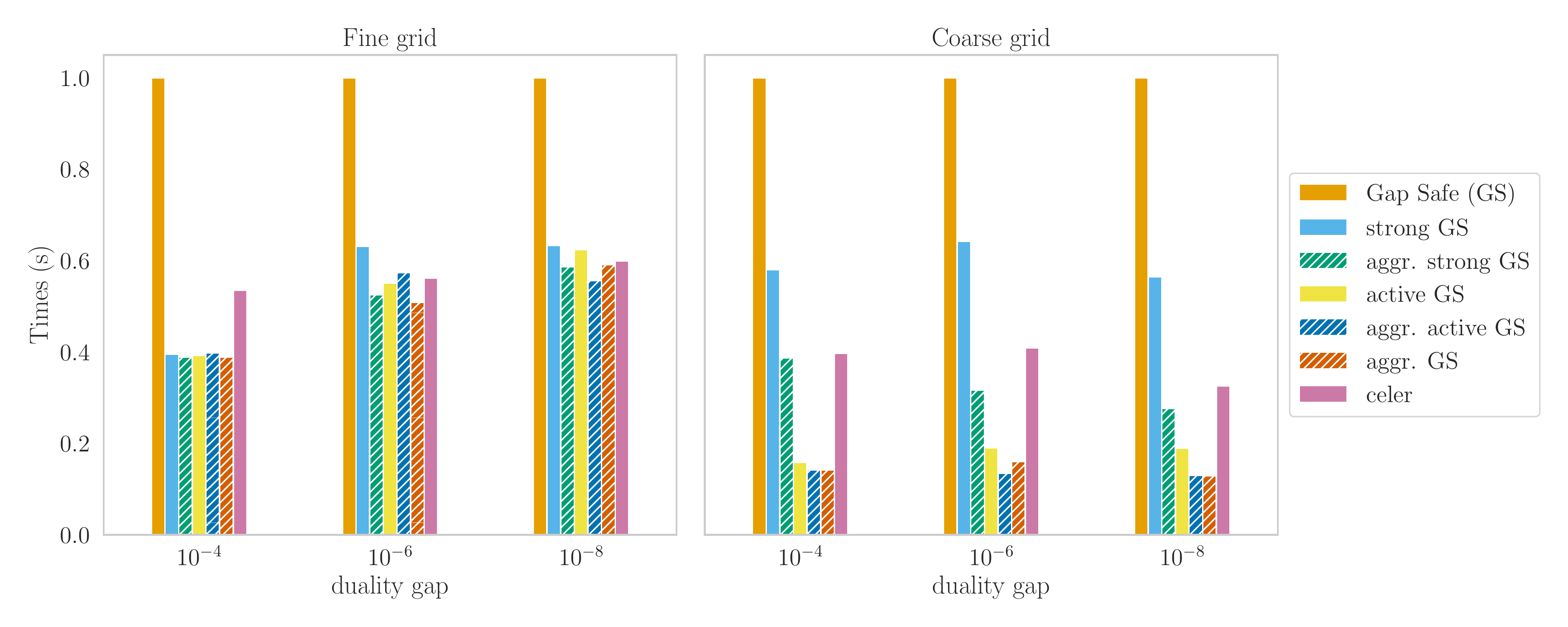}\label{subfig:lasso_climate}
\caption{Lasso on the \texttt{climate NCEP/NCAR Reanalysis 1} (dense data with $n=814$ observations and $p=73570$ features) see \cite{Kalnay_Kanamitsu_Kistler_Collins_Deaven_Gandin_Iredell_Saha_White_Woollen_Others96}. Computation times needed to solve the Lasso regression path to desired accuracy for a grid of $\lambda$ from $\lambda_{\max} = \norm{X^\top y}_{\infty}$ to $\lambda_{\max} / 100$. The size of the dense grid (resp. sparse grid) is $100$ (resp. 10).
}\label{fig:lasso_climate}
\end{figure}

\begin{figure}[H]
\includegraphics[width=\linewidth]{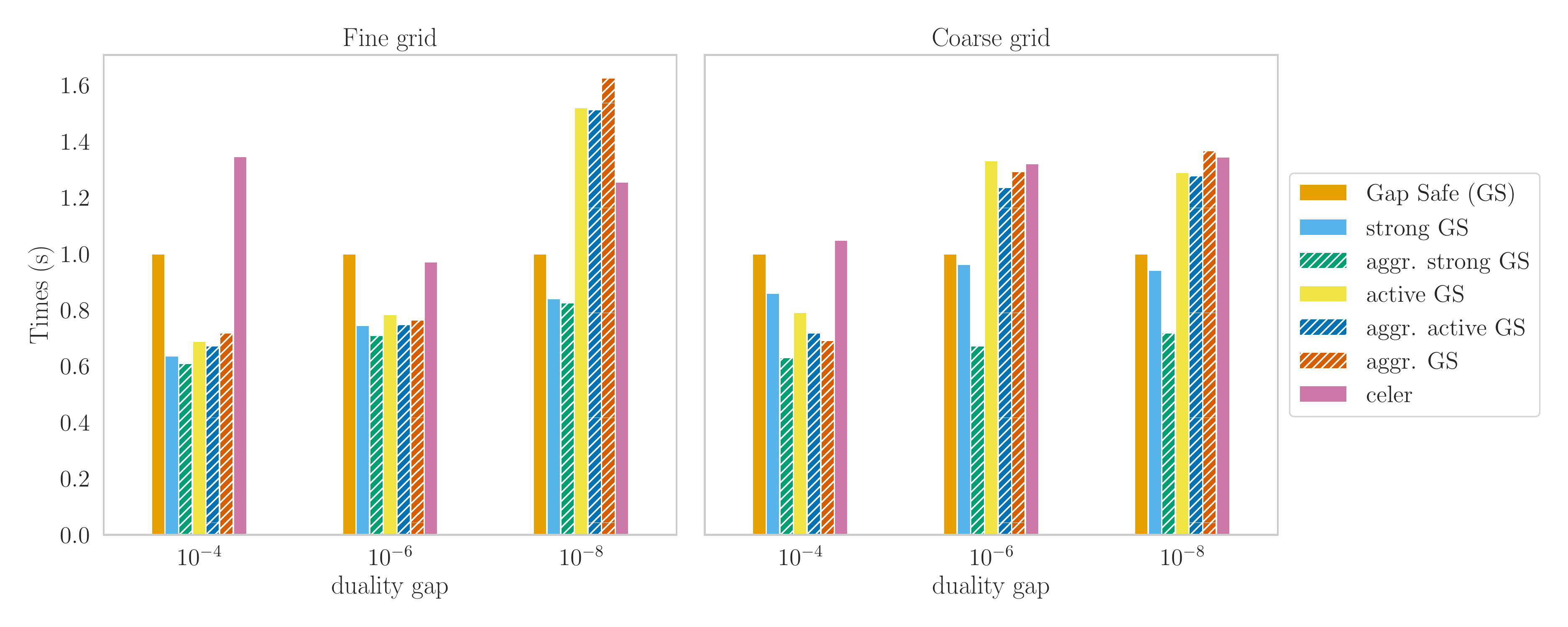}\label{subfig:lasso_rcv1_train}
\caption{Lasso on the \texttt{rcv1\_train} (sparse data with $n=20242$ observations and $p=19960$ features) available in libsvm. Computation times needed to solve the Lasso regression path to desired accuracy for a grid of $\lambda$ from $\lambda_{\max} = \norm{X^\top y}_{\infty}$ to $\lambda_{\max} / 100$. The size of the dense grid (resp. sparse grid) is $100$ (resp. 10).
}\label{fig:lasso_rcv1_train}
\end{figure}

\begin{figure}[H]
\includegraphics[width=\linewidth]{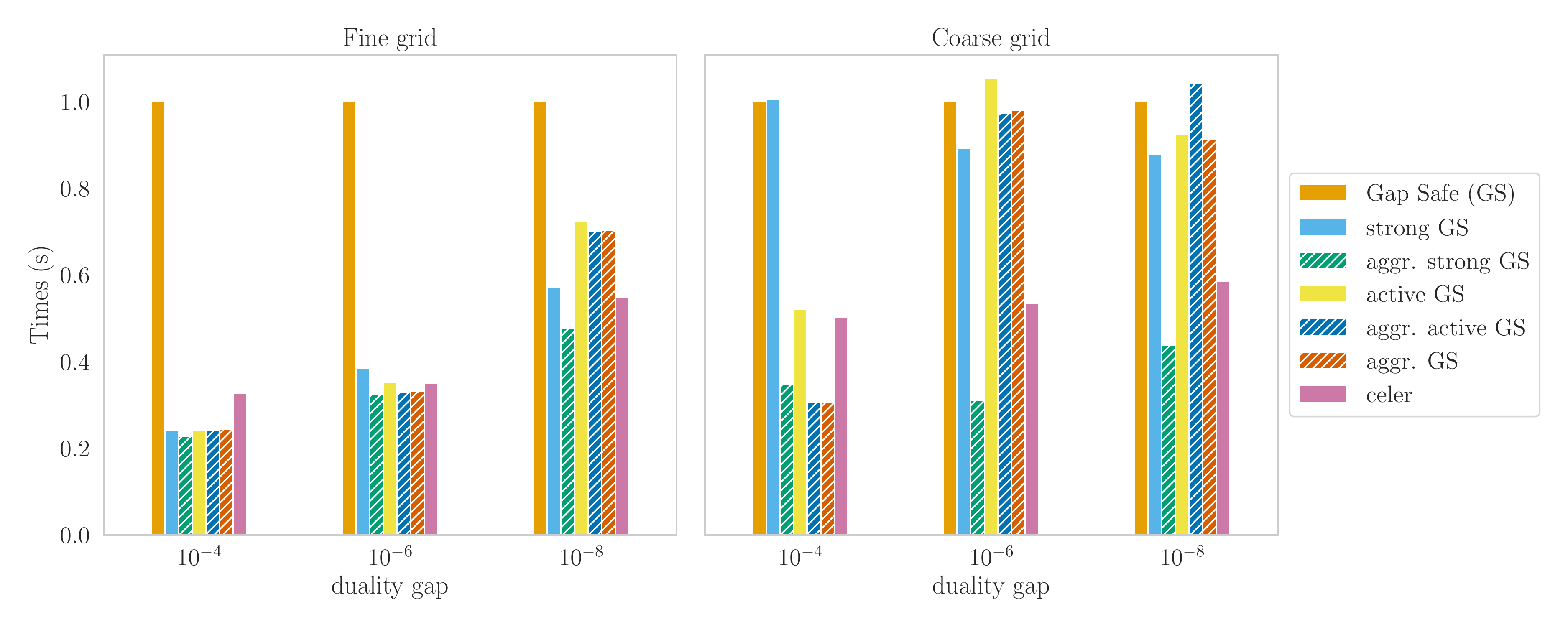}\label{subfig:lasso_news20}
\caption{Lasso on the \texttt{news20} (sparse data with $n=19996$ observations and $p=632983$ features) available in libsvm. Computation times needed to solve the Lasso regression path to desired accuracy for a grid of $\lambda$ from $\lambda_{\max} = \norm{X^\top y}_{\infty}$ to $\lambda_{\max} / 100$. The size of the dense grid (resp. sparse grid) is $100$ (resp. 10).
}\label{fig:lasso_news20}
\end{figure}

\section{Conclusion}

We have presented a simple way to unify various contributions that explicitly identify active variables, especially in sparse regression problems. For this, we have relied on optimality conditions and the fact that the subdifferentials of a function evaluated at two distinct points can not be overlapped. It should be noted that this remarkable property is not limited to convex functions (\eg it holds for non-convex setting as soon as the set \eqref{eq:definition_subdifferential} is non empty).

Extending the identification rules to subdifferential in the sense of Fr\'echet or Clarke would be a natural venue for future works. Promising results have been shown in \cite{Lee_Breheny15, Rakotomamonjy_Gasso_Salmon19}. However, it is still open to get a unified framework for non convex optimization problems and non separable regularization function.

When an optimization algorithms can benefit from screening rules, we have also shown that the number of iterations to identify the active set can be accurately estimated, and depends only on the rate of convergence of the (converging) algorithm used.
Numerical experiments of some heuristic acceleration rules have been provided, showing their interest for (block) coordinate descent algorithms.